\documentclass[letterpaper,12pt,hidelinks]{article}
\usepackage{figures/configuration, natbib%
}
\begin{document}
\title{Logistic Regression for Massive Data with Rare Events}
\author{HaiYing Wang\\
  Department of Statistics, University of Connecticut}
\maketitle
\begin{abstract}
This paper studies binary logistic regression for rare events data, or imbalanced data, where the number of events (observations in one class, often called cases) is significantly smaller than the number of nonevents (observations in the other class, often called controls). We first derive the asymptotic distribution of the maximum likelihood estimator (MLE) of the unknown parameter, which shows that the asymptotic variance convergences to zero in a rate of the inverse of the number of the events instead of the inverse of the full data sample size. This indicates that the available information in rare events data is at the scale of the number of events instead of the full data sample size. Furthermore, we prove that under-sampling a small proportion of the nonevents, the resulting under-sampled estimator may have identical asymptotic distribution to the full data MLE. This demonstrates the advantage of under-sampling nonevents for rare events data, because this procedure may significantly reduce the computation and/or data collection costs. Another common practice in analyzing rare events data is to over-sample (replicate) the events, which has a higher computational cost. We show that this procedure may even result in efficiency loss in terms of parameter estimation.
\end{abstract}

\section{Introduction}
Big data with rare events in binary responses, also called imbalanced data, are data in which the number of events (observations for one class of the binary response) is much smaller than the number of non-events (observations for the other class of the binary response). In this paper we also call the events ``cases'' and can the nonevents ``controls''. Rare events data are common in many scientific fields and applications. However, several important questions remain unanswered that are essential for valid data analysis and appropriate decision-making. For example, should we consider the amount of information contained in the data to be at the scale of the full-data sample size (very large) or the number of cases (relatively small)? 
Rare events data provide unique challenges and opportunities for sampling. On the one hand, sampling will not work without looking at responses because the probability of not selecting a rare case is high. On the other hand, since the rare cases are more informative than the controls, is it possible to use a small proportion of the full data to preserve most or all of the relevant information in the data about unknown parameters?
A common practice when analyzing rare events data is to under-sample 
the controls and/or over-sample (replicate) the cases. Is there any information loss when using this approach? %
This paper provides a rigorous theoretical analysis on the aforementioned questions in the context of parameter estimation. Some answers may be counter-intuitive. For example, keeping all the cases, there may be no efficiency loss at all for under-sampling controls; on the other hand, using all the controls and over-sampling cases may reduce estimation efficiency.

Rare events data, or imbalanced data, have attracted a lot of attentions in machine learning and other quantitative fields, such as \cite{Japkowicz2000,king2001logistic,chawla2004editorial,estabrooks2004multiple,owen2007infinitely,sun2007cost,chawla2009data,rahman2013addressing,fithian2014local,lemaitre2017imbalanced}. %
A commonly implemented approach in practice is to try balancing the data by under-sampling controls \citep{drummond2003c4,Zhou2009Exploratory} and/or over-sampling cases \citep{chawla2002smote,Han2005Borderline,mathew2017classification,douzas2017self}. 
However, most existing investigations focus on algorithms and methodologies for classification. Theoretical analyses of the effects of under-sampling and over-sampling in terms of parameter estimation are still rare. %

\cite{king2001logistic} considered logistic regression in rare events data and focused on correcting the biases in estimating the regression coefficients and probabilities. 
\cite{fithian2014local} utilized the special structure of logistic regression models to design a novel local case-control sampling method. These investigations obtained theoretical results based on the the regular assumption that the probability of event occurring is fixed and does not go to zero. This assumption rules out the scenario of extremely imbalanced data, because for extremely imbalanced data, it is more appropriate to assume that the event probability goes to zero. 
\cite{owen2007infinitely}'s investigation did not require this fixed-probability assumption. He assumed that the number of rare cases is fixed, and derived the non-trivial point limit of the slope parameter estimator in logistic regression. However, the convergence rate and distributional properties of this estimator were not investigated. 
In this paper, we obtain convergence rates and asymptotic distributions of parameter estimators under the assumption that both the number of cases and the number of controls are random, and they grow large in rates that the number of cases divided by the number of controls decays to zero. This is the first study that provides distributional results for rare events data with a decaying event rate, and it gives the following indications. %
\begin{itemize}
\item The convergence rate of the maximum likelihood estimator (MLE) is at the inverse of the number of cases instead of the total number of observations. This means that the amount of available information about unknown parameters in the data may be limited even the full data volume is massive.
\item There maybe no efficiency loss at all in parameter estimation if one removes most of the controls in the data, because the control under-sampled estimators may have an asymptotic distribution that is identical to that of the full data MLE. %
\item Besides higher computational cost, over-sampling cases may result in estimation efficiency loss, because the asymptotic variances of the resulting estimators may be larger than that of the full data MLE.
\end{itemize}

The rest of the paper is organized as follows. We introduce the model setup and related assumptions in Section~\ref{sec:model-setups-assumpt}, and derive the asymptotic distribution for the full data MLE. We investigate under-sampled estimators in Section~\ref{sec:effic-under-sampl} and study over-sampled estimators in Section~\ref{sec:efficiency-loss-due}. Section~\ref{sec:numer-demonstr} presents some numerical experiments, and Section~\ref{sec:disc-future-rese} concludes the paper and points out some necessary future research. All the proofs of theoretical findings in this paper are presented in the supplementary material. 
\black

\section{Model setups and assumptions}
\label{sec:model-setups-assumpt}
Let $\Dn=\{(\x_i,y_i), i=1, ..., n\}$ be independent data of size $n$ from a logistic regression model,
\begin{equation}\label{eq:1}
  \Pr(y=1|\x)=p(\alpha,\bbeta)
  =\frac{e^{\alpha+\x\tp\bbeta}}{1+e^{\alpha+\x\tp\bbeta}}.
\end{equation}
Here $\x\in\mathbb{R}^d$ is the covariate, $y\in\{0,1\}$ is the binary class label, $\alpha$ is the intercept parameter, and $\bbeta$ is the slope parameter vector. For ease of presentation, denote $\btheta=(\alpha,\bbeta\tp)\tp$ as the full vector of regression coefficient, and define $\z=(1,\x\tp)\tp$ accordingly. This paper focuses on estimating the unknown $\btheta$.

If $\btheta$ is fixed (does not change with $n$ changing), then model~(\ref{eq:1}) is just the regular logistic regression model, and classical likelihood theory shows that the MLE based on the full data $\Dn$ converges at a rate of $n^{-1/2}$. A fixed $\btheta$ implies that $\Pr(y=1)=\Exp\{\Pr(y=1|\x)\}$ is also a fixed constant bounded away from zero.  However, for rare events data, because the event rate is so low in the data, it is more appropriate to assume that $\Pr(y=1)$ approaches zero in some way. We discuss how to model this scenario in the following. 

Let $n_1$ and $n_0$ be the numbers cases (observations with $y_i=1$) and controls (observations with $y_i=0$), respectively, in $\Dn$. Here, $n_1$ and $n_0$ are random because they are summary statistics about the observed data, i.e., $n_1=\sumn y_i$ and $n_0=n-n_1$. For rare events data, $n_1$ is much smaller than $n_0$.  Thus, for asymptotic investigations, it is reasonable to assume that $n_1/n_0\rightarrow0$, or equivalently $n_1/n\rightarrow0$, in probability, as $n\rightarrow\infty$. For big data with rare events, there should be a fair amount of cases observed, so it is appropriate to assume that $n_1\rightarrow\infty$ in probability. To model this scenario, we assume that the marginal event probability $\Pr(y=1)$ satisfies that as $n\rightarrow\infty$,
\begin{equation}\label{eq:24}
  \Pr(y=1)\rightarrow0 \quad\text{and}\quad n\Pr(y=1)\rightarrow\infty.
\end{equation}
We accommodate this condition by assuming that the true value of $\bbeta$, denoted as $\bbeta_t$, is fixed while the true value of $\alpha$, denoted as $\alpha_t$, goes to negative infinity in a certain rate. Specifically, we assume $\alpha_t\rightarrow-\infty$ as $n\rightarrow\infty$ in a rate such that 
\begin{align}
  \frac{n_1}{n}&=\Pr(y=1)\{1+\op\}%
    =\Exp\bigg(\frac{e^{\alpha_t+\bbeta_t\tp\x}}
  {1+e^{\alpha_t+\bbeta_t\tp\x}}\bigg)\{1+\op\},\label{eq:2}
\end{align}
where $\op$ means a term that converges to zero in probability, i.e., a term that is arbitrarily small with probability approaching one. The assumption of a diverging $\alpha_t$ with a fixed $\bbeta_t$ means that the baseline probability of a rare event is low, and the effect of the covariate does not change the order of the probability for a rare event to occur. This is a very reasonable assumption for many practical problems. For example, although making phone calls when driving may increase the probability of car accidents, it may not make car accidents a high-probability event.

\subsection{How much information do we have in rare events data}
To demonstrate how much information is really available in rare events data, we derive the asymptotic distribution of the MLE for model~(\ref{eq:1}) in the scenario described in~(\ref{eq:24}) and \eqref{eq:2}. The MLE based on the full data $\Dn$, say $\hbeta$, is the maximizer of 
\begin{equation}
  \ell(\btheta)
  =\sumn\big\{y_i\z_i\tp\btheta-\log(1+e^{\z_i\tp\btheta})\big\},
\end{equation}
which is also the solution to the following equation,
\begin{equation}
  \dot\ell(\btheta)
  =\sumn\big\{y_i-p_i(\alpha,\bbeta)\big\}\z_i=0,
\end{equation}
where $\dot\ell(\btheta)$ is the gradient of the log-likelihood $\ell(\btheta)$. 

The following Theorem gives the asymptotic normality of the MLE  $\hbeta$ for rare events data. %
\begin{theorem}\label{thm:1}
  If $\Exp(e^{t\|\x\|})<\infty$ for any $t>0$ and $\Exp(e^{\bbeta_t\tp\x}\z\z\tp)$ is a positive-definite matrix, then under the conditions in (\ref{eq:24}) and (\ref{eq:2}), as $n\rightarrow\infty$,
  \begin{equation}\label{eq:3}
    \sqrt{n_1}(\htheta-\btheta_t)
    \longrightarrow \Nor\big(\0,\ \V_f\big),
  \end{equation}
  in distribution, where
  \begin{align}
    \V_f&=\Exp\big(e^{\bbeta_t\tp\x}\big)\M_{f}^{-1}, \qquad\text{and}\\
    \M_{f}&=\Exp\big(e^{\bbeta_t\tp\x}\z\z\tp\big)
                   =\Exp\left\{e^{\bbeta_t\tp\x}
  \begin{pmatrix}
    1 & \x\tp\\
    \x & \x\x\tp
  \end{pmatrix}\right\}.
\end{align}
\end{theorem}

\begin{remark}\normalfont
  The result in~(\ref{eq:3}) shows that the convergence rate of the full-data MLE is at the order of $n_1^{-1/2}$, i.e, $\htheta-\btheta_t=O_P(n_1^{-1/2})$. This is different from the classical result of $\htheta-\btheta_t=O_P(n^{-1/2})$ for the case that $\Pr(y=1)$ is a fixed constant.
  Theorem~\ref{thm:1} indicates that for rare events data, the real amount of available information is actually at the scale of $n_1$ instead of $n$. A large volume of data does not mean that we have a large amount of information. 
\end{remark}

\section{Efficiency of under-sampled estimators}
\label{sec:effic-under-sampl}
Theorem~\ref{thm:1} in the previous section shows that the full-data MLE has a convergence rate of $n_1^{-1/2}$.  %
If we under-sample controls to reduce the number of controls to the same level of $n_1$, whether the resulting estimator has the full-data estimator convergence rate of $n_1^{-1/2}$? If so, one can significantly improve the computational efficiency and reduce the storage requirement for massive data. Furthermore, will under-sampling controls causes any estimation efficiency loss (an enlarged asymptotic variance)? This section answers the aforementioned questions. 

From the full data set $\Dn=\{(\x_1,y_1), ..., (\x_n,y_n)\}$, we want to use all the cases (data points with $y_i=1$) while only select a subset for the controls (data points with $y_i=0$). Specifically, let $\pi_0$ be the probability that each data points with $y_i=0$ is selected in the subset. Let $\delta_i\in\{0, 1\}$ be the binary indicator variable that signifies if the $i$-th observation is included in the subset, i.e., include the $i$-th observation into the sample if $\delta_i=1$ and ignore the $i$-th observation if $\delta_i=0$. Here, we define the sampling plan by assigning
\begin{equation}\label{eq:16}
  \delta_i=y_i+(1-y_i)I(u_i\le \pi_0), \quad i=1, ..., n,
\end{equation}
where $u_i$$\sim \mathbb{U}(0,1)$, $i=1, ..., n$, are independent and identically distributed (i.i.d.) random variables with the standard uniform distribution. This is a mixture of deterministic selection and random sampling. The resulting control under-sampled data include all rare cases (with $y_i=1$) and the number of controls (with $y_i=0$) is on average at the order of $n_0\pi_0$. The average sample size for the under-sampled data given the full-data is $\sumn\Exp(\delta_i|\Dn)=n_1+n_0\pi_0$, 
which is $o_p(n)$ if $\pi_0\rightarrow0$. The average sample size reduction is $n_0(1-\pi_0)$ which is at the same order of $n$ if $\pi_0\nrightarrow1$, and $n_0(1-\pi_0)/n\rightarrow1$ if $\pi_0\rightarrow0$.

Note that the under-sampled data taken according to $\delta_i$ in~\eqref{eq:16} is a biased sample, so we need to maximize a weighted objective function to obtain an asymptotically unbiased estimator. Alternatively, we can maximize an unweighted objective function and then correct the bias for the resulting estimator in logistic regression.

\subsection{Under-sampled weighted estimator}
The sampling inclusion probability given the full data $\Dn$ for the $i$-th data point is
\begin{equation*}
  \pi_i=\Exp(\delta_i|\Dn)
  =y_i+(1-y_i)\pi_0=\pi_0+(1-\pi_0)y_i.
\end{equation*}
The under-sampled weighted estimator, $\htheta\sw$, is the maximizer of
\begin{align}\label{eq:8}
  \ell\sw(\btheta)=\sumn\frac{\delta_i}{\pi_i}
  \big\{y_i\z_i\tp\btheta-\log(1+e^{\z_i\tp\btheta})\big\}.
\end{align}

We present the asymptotic distribution of $\htheta\sw$ in the following theorem. 
\begin{theorem}\label{thm:2}
  If $\Exp(e^{t\|\x\|})<\infty$ for any $t>0$, $\Exp\big(e^{\btheta_t\tp\x}\z\z\tp\big)$ is a positive-definite matrix, and
  $c_n=e^{\alpha_t}/\pi_0\rightarrow c$ for a constant $c\in[0,\infty)$, then under the conditions in (\ref{eq:24}) and (\ref{eq:2}), as $n\rightarrow\infty$,
  \begin{equation}
    \sqrt{n_1}(\htheta\sw-\btheta_t)
    \longrightarrow \Nor(\0,\ \V\sw),
  \end{equation}
  in distribution, where 
  \begin{align}
    \V\sw&=\Exp(e^{\bbeta_t\tp\x})\M_{f}^{-1}\M\sw\M_{f}^{-1},
    \quad\text{and}\\
    \M\sw&=\Exp\big\{e^{\bbeta_t\tp\x}(1+ce^{\bbeta_t\tp\x})\z\z\tp\big\}.
\end{align}
\end{theorem}

\begin{remark}\normalfont
  If $\Exp(e^{t\|\x\|})<\infty$ for any $t>0$, then from (\ref{eq:2}) and the dominated convergence theorem, we know that $n_1=ne^{\alpha_t}\Exp(e^{\bbeta_t\tp\x})\{1+\op\}$. Thus 
  \begin{equation*}
    c_n\Exp(e^{\bbeta_t\tp\x})=\frac{n_1}{n\pi_0}\{1+\op\}
    =\frac{n_1}{n_0\pi_0}\{1+\op\}.
  \end{equation*}
Since $n_0\pi_0$ is the average number of the controls in the under-sampled data, $c\Exp(e^{\bbeta_t\tp\x})$ can be interpreted as the asymptotic ratio of the number of cases to the number of controls in the under-sampled data. Therefore, since $\Exp(e^{\bbeta_t\tp\x})>0$ is a fixed constant, the value of $c$ has the following intuitive %
interpretations.
  \begin{itemize}
  \item $c=0$: take much more controls than cases;
  \item $0<c<\infty$: the number of controls to take is at the same order of the number of cases;
  \item $c=\infty$: take much fewer controls than cases.
  \end{itemize}
Theorem~\ref{thm:2} requires that $0\le c<\infty$. This means that the number of controls to take should not be significantly smaller than the number of cases, which is a very reasonable assumption. 
\end{remark}

\begin{remark}\normalfont
  Theorem~\ref{thm:2} shows that as long as $\pi_0$ does not make the number of controls in the under-sampled data much smaller than the number of cases $n_1$, then the under-sampled estimator $\htheta\sw$ preserves the convergence rate of the full-data estimator. Furthermore, if $c=0$ then $\M\sw=\M_{f}$, which implies that $\V\sw=\V_f$. This means that if one takes much more controls than cases, then asymptotically there is no estimation efficiency loss at all. Here, the number of controls to take can still be significantly smaller than $n_0$ so that the computational burden is significantly reduced. If $c>0$, since $\M\sw>\M_{f}$, we know that $\V\sw>\V_f$, in the Loewner order\footnote{For two Hermitian matrices $\A_1$ and $\A_2$ of the same dimension, $\A_1\ge\A_2$ if $\A_1-\A_2$ is positive semi-definite and $\A_1>\A_2$ if $\A_1-\A_2$ is positive definite.}. Thus reducing the number of controls to the same order of the number of cases may reduce  %
the estimation efficiency, %
although the convergence rate is the same as that of the full-data estimator.
\end{remark}

\subsection{Under-sampled unweighted estimator with bias correction}
Based on the control under-sampled data, if we obtain an estimator from an unweighted objective function, say
\begin{align*}
  \ttheta\su&=\arg\max_{\btheta}\ \ell\su(\btheta)%
    =\arg\max_{\btheta}\sumn\delta_i
  \big[y_i\z_i\tp\btheta-\log\{1+e^{\z_i\tp\btheta}\}\big],
\end{align*}
then in %
$\ttheta\su=(\hat\alpha\su,\hbeta\su\tp)\tp$, the intercept estimator $\hat\alpha\su$ is asymptotically biased while the slope estimator $\hbeta\su$ is still asymptotically unbiased. We correct the bias of $\hat\alpha\su$ using $\log(\pi_0)$, and define the under-sampled {\bf u}nweighted estimator with {\bf b}ias {\bf c}orrection $\htheta\sbc$ as
\begin{equation}
  \htheta\sbc=\ttheta\su+\b,
\end{equation}
where
\begin{equation}
  \b=\{\log(\pi_0), 0, ..., 0\}\tp.
\end{equation}
The following theorem gives the asymptotic distribution of $\htheta\sbc$. 
\begin{theorem}\label{thm:3}
  If $\Exp(e^{t\|\x\|})<\infty$ for any $t>0$, $\Exp\big(e^{\btheta_t\tp\x}\z\z\tp\big)$ is a positive-definite matrix, and
  $e^{\alpha_t}/\pi_0\rightarrow c$ for a constant $c\in[0,\infty)$, then under the conditions in (\ref{eq:24}) and (\ref{eq:2}), as $n\rightarrow\infty$,
  \begin{equation}
    \sqrt{n_1}(\htheta\sbc-\btheta_t)
    \longrightarrow \Nor(\0,\ \V\sbc),
  \end{equation}
  in distribution, where
  \begin{align}
    \V\sbc&=\Exp(e^{\bbeta_t\tp\x})(\M\sbc)^{-1}, \quad\text{and}\\
    \M\sbc&=\Exp\bigg(\frac{e^{\bbeta_t\tp\x}}{1+ce^{\bbeta_t\tp\x}}
     \z\z\tp\bigg).
\end{align}
\end{theorem}

\begin{remark}\normalfont
  Similarly to the case of under-sampled weighted estimator, Theorem~\ref{thm:3} shows that the estimator $\htheta\sbc$ preserves the same convergence rate of the full-data estimator if $c<\infty$. Furthermore, if $c>0$ then $\V\sbc>\V_f$; if $c=0$, then $\V\sbc=\V_f$.
\end{remark}

The following proposition is useful to compare the asymptotic variances of the weighted and the unweighted estimators. 
\begin{proposition}\label{prop1}
  Let $\v$ be a random vector and $h$ be a positive scalar random variable. Assume that
$\Exp(\v\v\tp)$, $\Exp(h\v\v\tp)$, and $\Exp(h^{-1}\v\v\tp)$ are all finite and positive-definite matrices. The following inequality holds in the Loewner order.
  \begin{equation*}
    \big\{\Exp(h^{-1}\v\v\tp)\big\}^{-1}\le
    \big\{\Exp(\v\v\tp)\big\}^{-1}\Exp(h\v\v\tp)
    \big\{\Exp(\v\v\tp)\big\}^{-1}.
  \end{equation*}
\end{proposition}

\begin{remark}\normalfont
  If we let $\v=e^{\bbeta_t\tp\x/2}\z$ and $h=1+ce^{\bbeta_t\tp\x}$ in Proposition~\ref{prop1}, then we know that $\V\sbc\le\V\sw$ in the Loewner order. This indicates that with the same control under-sampled data, the unweighted estimator with bias correction, $\htheta\sbc$, has a higher estimation efficiency than the weighted estimator, $\htheta\sw$. 
\end{remark}

\section{Efficiency loss due to over-sampling}
\label{sec:efficiency-loss-due}
Another common practice to analyze rare events data is to use all the controls and over-sample the cases. To investigate the effect of this approach, let $\tau_i$ denote the number of times that a data point is used, and define 
\begin{equation}\label{eq:19}
  \tau_i%
  =y_iv_i+1,\quad i=1, ...,n,
\end{equation}
where $v_i\sim \mathbb{POI}(\lambda_n)$, $i=1, ..., n$, are i.i.d. Poisson random variables with parameter $\lambda_n$. 
For this over-sampling plan, a data point with $y_0=0$ will be used only one time, while a data point with $y_i=1$ will be on average used in the over-sampled data for $\Exp(\tau_i|\Dn,y_i=1)=1+\lambda_n$ times. Here, $\lambda_n$ can be interpreted as the average over-sampling rate for cases. 

Again, the case over-sampled data according to \eqref{eq:19} is a biased sample, and we need to use a weighted objective function or to correct the bias of the estimator form an unweighted objective function.

\subsection{Over-sampled weighted estimator}
Let $w_i=\Exp(\tau_i|\Dn)=1+\lambda_ny_i$.  The case over-sampled {\bf w}eighted estimator, $\htheta\ow$, is the maximizer of
\begin{equation}\label{eq:11}
  \ell\ow(\btheta)=\sumn\frac{\tau_i}{w_i}
  \big\{y_i\z_i\tp\btheta-\log(1+e^{\z_i\tp\btheta})\big\}.
\end{equation}

The following theorem gives the asymptotic distribution of $\htheta\ow$.
\begin{theorem}\label{thm:4}
  If $\Exp(e^{t\|\x\|})<\infty$ for any $t>0$, $\Exp\big(e^{\btheta_t\tp\x}\z\z\tp\big)$ is positive-definite, and $\lambda_n\rightarrow\lambda\ge0$,
  then under the conditions in (\ref{eq:24}) and (\ref{eq:2}), as $n\rightarrow\infty$,
  \begin{equation}
    \sqrt{n_1}(\htheta\ow-\btheta_t)
    \longrightarrow \Nor(\0, \V\ow),
  \end{equation}
  in distribution, where
  \begin{equation}\label{eq:4}
    \V\ow=\frac{(1+\lambda)^2+\lambda}{(1+\lambda)^2}
    \Exp(e^{\bbeta_t\tp\x})\M_{f}^{-1}.
  \end{equation}
\end{theorem}

\begin{remark}
  Note that in (\ref{eq:4}), $\frac{(1+\lambda)^2+\lambda}{(1+\lambda)^2}\ge1$ and the equality holds only if $\lambda=0$ or $\lambda=\infty$. Thus, $\V\ow\ge\V_f$, meaning that over-sampling the cases may result in estimation efficiency loss unless the number of over-sampled cases is small enough to be negligible ($\lambda=0$) or it is very large ($\lambda=\infty$). Considering that over-sampling incurs additional computational cost with potential estimation efficiency loss, this procedure is not recommended if the primary goal is parameter estimation. 
\end{remark}

\subsection{Over-sampled unweighted estimator with bias correction}
For completeness, we derive the asymptotic distribution of the over-sampled unweighted estimator with {\bf b}ias {\bf c}orrection, $\htheta\obc$, defined as $\htheta\obc=\ttheta\ou-\b_o$, where
\begin{align}
  \ttheta\ou&=\arg\max_{\btheta}\ \ell\ou(\btheta)%
    =\arg\max_{\btheta}\sumn
  \tau_i\big[y_i\z_i\tp\btheta-\log\{1+e^{\z_i\tp\btheta}\}\big],
\end{align}
and %
\begin{equation}
  \b_o=(b_{o0}, 0, ..., 0)\tp=\{\log(1+\lambda_n), 0, ..., 0\}\tp.
\end{equation}
The following theorem is about the asymptotic distribution of $\htheta\obc$. 
\begin{theorem}\label{thm:5}
  If $\Exp(e^{t\|\x\|})<\infty$ for any $t>0$, $\Exp\big(e^{\btheta_t\tp\x}\z\z\tp\big)$ is positive-definite, $\lambda_n\rightarrow\lambda\ge0$, and $\lambda_ne^{\alpha_t}\rightarrow \cm$ for a constant $\cm\in[0,\infty)$, then under the conditions in (\ref{eq:24}) and (\ref{eq:2}), as $n\rightarrow\infty$,
  \begin{equation}
    \sqrt{n_1}(\htheta\obc-\btheta_t)
    \longrightarrow \Nor(\0,\ \V\obc),
  \end{equation}
  in distribution, where
  \begin{align*}
    &\V\obc=\frac{(1+\lambda)^2+\lambda}{(1+\lambda)^2}
      \Exp\big(e^{\bbeta_t\tp\x}\big)
      \M_{obc2}^{-1}\M_{obc1}\M_{obc2}^{-1},\\
    &\M_{obc1}=\Exp\bigg\{\frac{e^{\bbeta_t\tp\x}}
    {(1+\cm e^{\bbeta_t\tp\x})^2}\z\z\tp\bigg\}, \quad\text{ and }\quad\\
  &\M_{obc2}=\Exp\bigg(\frac{e^{\bbeta_t\tp\x}}
    {1+\cm e^{\bbeta_t\tp\x}}\z\z\tp\bigg).
\end{align*}
\end{theorem}

\begin{remark}\normalfont
 Unlike the case of under-sampled estimators, for over-sampled estimators, the unweighted estimator with bias correction $\htheta\obc$ has a lower estimation efficiency than the weighted estimator $\htheta\ow$. To see this, letting $h=(1+\cm e^{\bbeta_t\tp\x})^{-1}$ and $\v=e^{\bbeta_t\tp\x/2}(1+\cm e^{\bbeta_t\tp\x})^{-1/2}\z$ in Proposition~\ref{prop1}, we know that $\V\obc\ge\V\ow$, and the equality holds if $c_o=0$. Here, since $\lambda_ne^{\alpha_t}\Exp(e^{\bbeta_t\tp\x})=\frac{n_1\lambda_n}{n_0}\{1+\op\}$, we can intuitively interpret $\cm\Exp(e^{\bbeta_t\tp\x})$ as the ratio of the average times of over-sampled cases to the number of controls. If in addition $\lambda=0$, then $\V\obc=\V\ow=\V_f$; but in general, $\V\obc\ge\V\ow\ge\V_f$. 
\end{remark}

\begin{remark}\normalfont
  Compared with Theorem~\ref{thm:4} for $\htheta\sw$, Theorem~\ref{thm:5} for $\htheta\obc$ requires an extra condition that $\lambda_ne^{\alpha_t}\rightarrow \cm\in[0,\infty)$. In addition, $\V\obc\ge\V\ow$. Thus, if over-sampling has to be implemented, then we recommend using the weighted estimator $\htheta\ow$. 
\end{remark}

\section{Numerical experiments}
\label{sec:numer-demonstr}
\subsection{Full data estimator $\htheta$}
Consider model~\eqref{eq:1} with one covariate $x$ and $\btheta=(\alpha,\beta)\tp$. We set $\Pr(y=1)=0.02$, $0.004$, $0.0008$ and $0.00016$, and generate corresponding full data of sizes $n=10^3$, $10^4$, $10^5$ and $10^6$, respectively. As a result, the average numbers of cases ($y_i=1$) in the resulting data are $\Exp(n_1)=20$, $40$, $80$ and $160$. The above value configuration aims to mimic the scenario that $n\rightarrow\infty$, $\Pr(y=1)\rightarrow0$, and $\Exp(n_1)\rightarrow\infty$. The covariates $x_i$'s are generated from $\Nor(1,1)$ for cases ($y_i=1$) and from $\Nor(0,1)$ for controls ($y_i=0$). For the above setup, the true value of $\beta$ is fixed $\beta_t=1$, and the true values of $\alpha$ are $\alpha_t=-4.39$, $-6.02$, $-7.63$ and $-9.24$, respectively for the four different values of $n$. 
We repeat the simulation for $S=1,000$ times and calculate empirical MSEs as $\text{eMSE}(\hat\theta_j)=S^{-1}\sum_{s=1}^S
(\hat\theta_j^{(s)}-\theta_{tj})^2$, $j=0,1$, where $\hat\theta_0=\hat\alpha$, $\hat\theta_1=\hat\beta$, and $\hat\theta_j^{(s)}$ is the estimate from the $s$-th repetition.

Table~\ref{tab:1} presents empirical MSEs (eMSEs) multiplied by $\Exp(n_1)$ and $n$, respectively. We see that $\Exp(n_1)\times$eMSE$(\hat\theta_j)$ does not diverge as $n$ increases for both $\hat\alpha$ and $\hat\beta$. This confirms the conclusion in Theorem~\ref{thm:1} that $\htheta$ converges at a rate of $n_1^{-1/2}$ (It implies that $n_1\|\htheta-\btheta_t\|^2=\Op$). On the other hand, values of $n\times$eMSE$(\hat\theta_j)$ are large, and they increase fast as $n$ increases, indicating that $n\|\htheta-\btheta_t\|^2$ diverges to infinity. Table~\ref{tab:1} confirms that although the values of the full data sample sizes $n$ are very large, it is the values of $n_1$ that reflect the real amount of available information about regression parameters, and they are actually much smaller.

\begin{table}[htbp]
  \caption{Empirical MSE (eMSE) multiplied by $\Exp(n_1)$ and $n$. }
  \label{tab:1}
\vskip 0.15in
\begin{center}
\begin{tabular}{ccrrcrr}\hline
  $n$ & $\Exp(n_1)$ & \multicolumn{2}{c}{$\Exp(n_1)\times$eMSE$(\hat\theta_j)$} &  & \multicolumn{2}{c}{$n\times$eMSE$(\hat\theta_j)$}\\
  \cline{3-4}\cline{6-7}
 & & $\hat\alpha$ & $\hat\beta$ &  & $\hat\alpha$ & $\hat\beta$ \\
\hline
$10^3$ & 20 & 2.51 & 1.21 &  & 125.7 & 60.6 \\
$10^4$ & 40 & 2.06 & 1.09 &  & 515.5 & 271.9 \\
$10^5$ & 80 & 2.22 & 1.00 &  & 2774.4 & 1248.8 \\
$10^6$ & 160 & 2.16 & 1.08 &  & 13474.9 & 6731.6 \\
\hline
\end{tabular}
\end{center}
\vskip -0.1in
\end{table}

\subsection{Sampling-based estimators} 
Now we provide numerical results about under-sampled and over-sampled estimators. 
Consider model~\eqref{eq:1} with $n=10^5$, $x\sim \Nor(0,1)$, and $\btheta_t=(-6,1)\tp$, so that $\Pr(y=1)\approx0.004$. For under-sampling, consider $\pi_0=0.005$, $0.01$, $0.05$, $0.1$, $0.2$, $0.5$, $0.8$, and $1.0$; for over-sampling, consider $\lambda_n=0, 0.22, 0.49, 1.23, 3.48, 6.39, 11.18$ and $53.6$, which corresponds to $\log(1+\lambda_n)=0$, $0.2$, $0.4$, $0.8$, $1.5$, $2.0$, $2.5$ and $4.0$, respectively. We repeat the simulation for $S=1,000$ times and calculate empirical MSEs as
\begin{equation*}
  \text{eMSE}(\htheta_g)=\frac{1}{S}\sum_{s=1}^S
  \|\htheta_g^{(s)}-\btheta_t\|^2,
\end{equation*}
where $\htheta_g^{(s)}$ is the estimate from the $s$-th repetition for some estimator $\htheta_g$. We consider $\htheta_g=\hbeta\sw$, $\hbeta\sbc$, $\hbeta\ow$, and $\hbeta\obc$. 
Note that if $\pi_0=1$ then the under-sampled estimators become the full data estimator, i.e., $\hbeta\sw=\hbeta\sbc=\hbeta$; if $\lambda_n=0$, then the over-sampled estimators become the full data estimator, i.e., $\hbeta\ow=\hbeta\obc=\hbeta$.

Figure~\ref{fig:so} presents the simulation results. Figure~\ref{fig:so} (a) plots eMSEs ($\times10^3$) against $\pi_0$. When $\pi_0$ is small, the number of controls in under-sampled data is small, and the resulting estimators are not as efficient as the full-data estimator. For example, when $\pi_0=0.005$, the numbers of cases and the numbers of controls are roughly the same, and we do see significant information loss in this case. However, when $\pi_0$ gets larger, under-sampled estimators becomes more efficient, and when $\pi_0>0.1$, the performances of the under-sampled estimators are almost as good as the full-data estimator. In addition, the unweighted estimator $\hbeta\sbc$ is more efficient than the weighted estimator $\hbeta\sw$ for smaller $\pi_0$'s, and they both perform more similarly to the full data estimator $\hbeta$ as $\pi_0$ grows. These observations are consistent with the conclusions in Theorems~\ref{thm:2} and \ref{thm:3}, and the discussions in the relevant remarks. 

Figure~\ref{fig:so} (b) plots eMSEs ($\times10^3$) against $\log(\lambda_n+1)$. We see that the case over-sampled estimators are less efficient than the full data estimator unless the average number of over-sampled cases $\lambda_n$ is very small or very large. For small $\lambda_n$, $\hbeta\ow$ and $\hbeta\obc$ perform similarly, but $\hbeta\ow$ is more efficient than $\hbeta\obc$ for large $\lambda_n$. The reason of this phenomenon is that if $\lambda_n$ is large, then the required condition of $\lambda_ne^{\alpha_t}\rightarrow \cm\in[0,\infty)$ in Theorem~\ref{thm:5} for $\htheta\obc$ may not be valid. This confirms our recommendation that the weighted estimator $\htheta\ow$ is preferable if over-sampling has to be used.

\begin{figure}[H]%
  \centering
  \begin{subfigure}{0.485\textwidth}
    \includegraphics[width=\textwidth,page=2]{subover.pdf}
    \caption{eMSEs ($\times10^3$) for under-sampling}
  \end{subfigure}
  \begin{subfigure}{0.485\textwidth}
    \includegraphics[width=\textwidth,page=1]{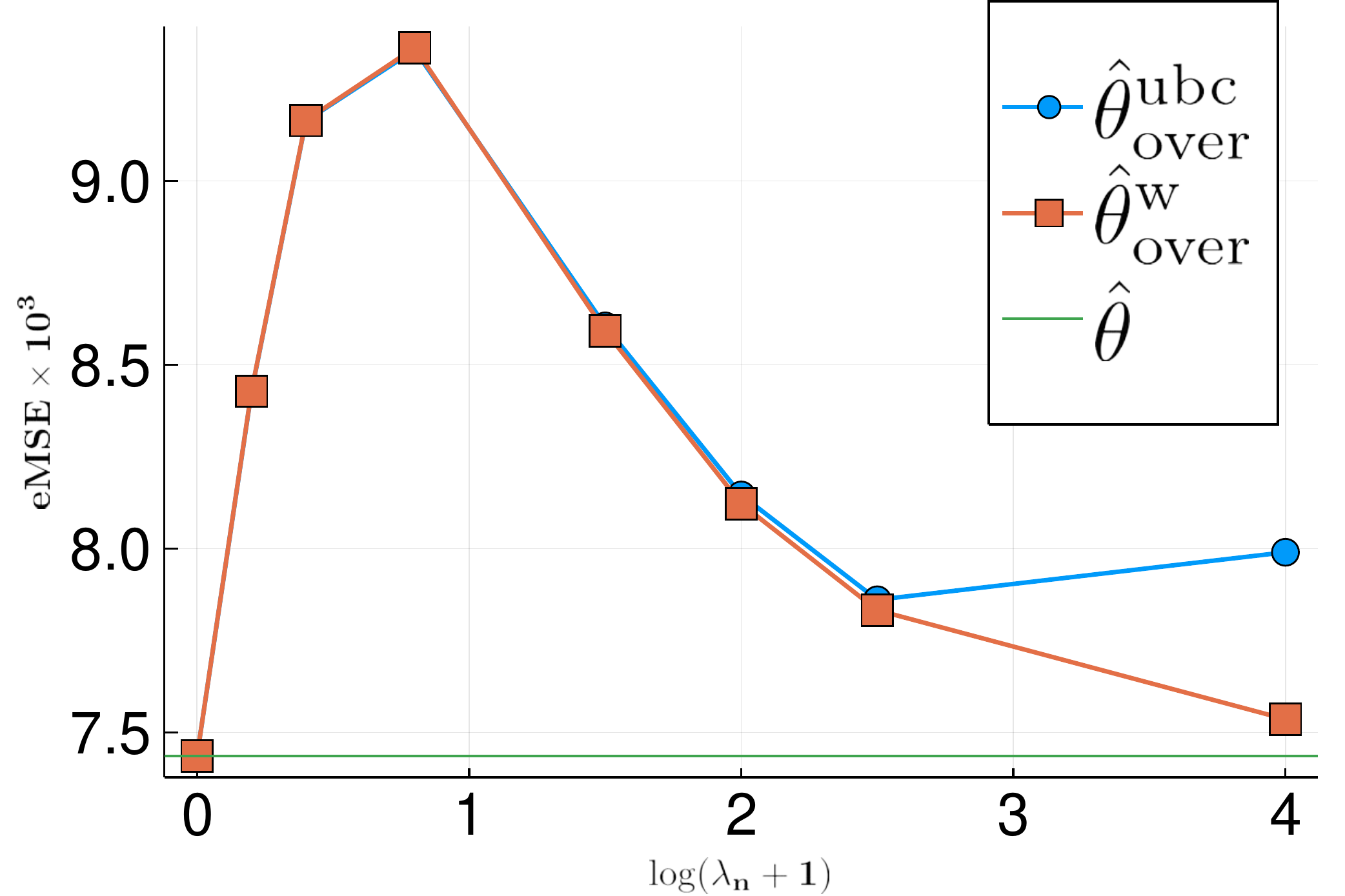}
    \caption{eMSE for over-sampling}
  \end{subfigure}
  \caption{Empirical MSEs ($\times10^3$) of under-sampled and over-sampled estimators. The eMSE ($\times10^3$) for the full data estimator $\htheta$ (the horizontal line) is also plotted for comparison. A smaller eMSE means that the corresponding estimator has a higher estimation efficiency.}
  \label{fig:so}
\end{figure}

\section{Discussion and future research}
\label{sec:disc-future-rese}
In this paper, we have obtained distributional results showing that the amount of information contained in massive data with rare events is at the scale of the relatively small total number of cases rather than the large total number of observations. We have further demonstrated that aggressively under-sampling the controls may not sacrifice the estimation efficiency at all while over-sampling the cases may reduce the estimation efficiency. 

Although the current paper focuses on the logistic regression model, we conjecture that our conclusions are generally true for rare events data and will investigate more complicated and general models in future research projects. As another direction, more comprehensive numerical experiments are helpful to gain further understandings on parameter estimation with imbalanced data. This paper has focused on point estimation. How to make valid and more accurate statistical inference with rare events data still need further research. There is a long standing literature investigating the effects of under-sampling and over-sampling in classification. However, most investigations adopted an empirical approach, so theoretical investigations on the effects of sampling are still needed for classification.

\newpage\appendix
\begin{center} \LARGE Appendix\\[1cm] \end{center}
\setcounter{equation}{0}
\renewcommand{\theequation}{A.\arabic{equation}}
\renewcommand{\thesection}{A.\arabic{section}}
\renewcommand{\thefigure}{A.\arabic{figure}}

In this section, we give prove all theoretical results in the paper. To facilitate the presentation of the proof, %
denote
\begin{equation*}
  a_n=\sqrt{ne^{\alpha_t}}.
\end{equation*} 
The condition that $\Exp(e^{t\|\x\|})<\infty$ for any $t>0$ implies that
\begin{equation}\label{eq:s18}
  \Exp(e^{t_1\|\x\|}\|\z\|^{t_2})<\infty,
\end{equation}
for any $t_1>0$ and $t_2>0$, and we will use this result multiple times in the proof. The inequality in~(\ref{eq:s18}) is true because for any $t_1>0$ and $t_2>0$, we can choose $t>t_1$ and $k>t_2$ so that
\begin{align*}
  e^{t\|\x\|}\ge e^{-t}e^{t\|\z\|}
  =e^{-t}e^{t_1\|\z\|}e^{(t-t_1)\|\z\|}
  \ge\frac{(t-t_1)^ke^{-t}}{k!}e^{t_1\|\x\|}\|\z\|^k
  \ge\frac{(t-t_1)^ke^{-t}}{k!}e^{t_1\|\x\|}\|\z\|^{t_2}.
\end{align*}
 with probability one.

\section{Proof of Theorem~\ref{thm:1}}
\begin{proof}[Proof of Theorem~\ref{thm:1}]
  The estimator $\htheta$ is the maximizer of 
\begin{align}
  \ell(\btheta)
  =&\sumn\big[
     (\alpha+\x_i\tp\bbeta)y_i
     -\log\{1+\exp(\alpha+\x_i\tp\bbeta)\}\big],
\end{align}
so $\u_n=a_n(\hat\btheta-\btheta_t)$ 
is the maximizer of
\begin{align}
  \gamma(\u)=\ell(\btheta_t+a_n^{-1}\u)-\ell(\btheta_t).
\end{align}
By Taylor's expansion,
\begin{align}
  \gamma(\u)
  &=a_n^{-1}\u\tp\dot\ell(\btheta_t)
    +0.5a_n^{-2}\sumn\phi_i(\btheta_t+a_n^{-1}\Acute\u)(\z_i\tp\u)^2,
\end{align}
where $\phi_i(\btheta)=p_i(\alpha,\bbeta)\{1-p_i(\alpha,\bbeta)\}$, and %
\begin{equation*}
  \dot\ell(\btheta)=\frac{\partial\ell(\btheta)}{\partial\btheta}
  =\sumn\{y_i-p_i(\btheta)\}\z_i=\sumn\{y_i-p_i(\alpha,\bbeta)\}\z_i
\end{equation*}
is the gradient of $\ell(\btheta)$, and $\Acute\u$ lies between $\0$ and $\u$.
If we can show that
\begin{align}\label{eq:s3}
  a_n^{-1}\dot\ell(\btheta_t)
  \longrightarrow \Nor\big(\0,\ \M_{f}\big),
\end{align}
in distribution, and for any $\u$,
\begin{align}\label{eq:s4}
  a_n^{-2}\sumn\phi_i(\btheta_t+a_n^{-1}\Acute\u)\z_i\z_i\tp
 \longrightarrow \M_{f},
\end{align}
in probability, then from the Basic Corollary in page 2 of \cite{hjort2011asymptotics}, we know that $a_n(\hat\btheta-\btheta_t)$, the maximizer of $\gamma(\u)$, satisfies that
\begin{equation}\label{eq:s1}
  a_n(\hat\btheta-\btheta_t)
  =\M_{f}^{-1}\times a_n^{-1}\dot\ell(\btheta_t)+\op.
\end{equation}
Slutsky's theorem together with (\ref{eq:s3}) and (\ref{eq:s1}) implies the result in Theorem~\ref{thm:1}. We prove \eqref{eq:s3} and \eqref{eq:s4} in the following. 

Note that %
\begin{equation}
  \dot\ell(\btheta_t)
  =\sumn\big\{y_i-p_i(\alpha_t,\bbeta_t)\big\}\z_i=0,
\end{equation}
is a summation of i.i.d. quantities. Since $\alpha_t\rightarrow-\infty$ as $n\rightarrow\infty$, the distribution of $\{y-p(\alpha_t,\bbeta_t)\}\z$ depends on $n$, we need to use a central limit theorem for triangular arrays. The Lindeberg-Feller central limit theorem \citep[see, Section $^*$2.8 of][]{Vaart:98} is appropriate. 

We exam the mean and variance of $a_n^{-1}\dot\ell(\btheta_t)$. For the mean, from the fact that
\begin{align*}
  \Exp[\{y_i-p_i(\alpha_t,\bbeta_t)\}\z_i]
  =\Exp[\Exp\{y_i-p_i(\alpha_t,\bbeta_t)|\z_i\}\z_i]
  =\0,
\end{align*}
we know that $\Exp\{a_n^{-1}\dot\ell(\btheta_t)\}=\0$.

For the variance,
\begin{align*}
  \Var\{a_n^{-1}\dot\ell(\btheta_t)\}
  &=a_n^{-2}\sumn\Var[\{y_i-p_i(\alpha_t,\bbeta_t)\}\z_i]
  =a_n^{-2}n\Exp\{\phi(\btheta_t)\z\z\tp\}\\
  &=a_n^{-2}n\Exp\bigg\{\frac{e^{\alpha_t+\bbeta_t\tp\x}\z\z\tp}
     {(1+e^{\alpha_t+\bbeta_t\tp\x})^2}\bigg\}
  =\Exp\bigg\{\frac{e^{\bbeta_t\tp\x}\z\z\tp}
     {(1+e^{\alpha_t+\bbeta_t\tp\x})^2}\bigg\}.
\end{align*}
Note that
\begin{align*}
  \frac{e^{\bbeta_t\tp\x}\z\z\tp}{(1+e^{\alpha_t+\bbeta_t\tp\x})^2}
  \longrightarrow e^{\bbeta_t\tp\x}\z\z\tp,
\end{align*}
almost surely, and
\begin{align*}
  \frac{e^{\bbeta_t\tp\x}\|\z\|^2}{(1+e^{\alpha_t+\bbeta_t\tp\x})^2}
  \le e^{\bbeta_t\tp\x}\|\z\|^2
  \quad\text{ with }\quad
  \Exp(e^{\bbeta_t\tp\x}\|\z\|^2)\le\infty.
\end{align*}
Thus, from the dominated convergence theorem,
\begin{align*}
  \Var\{a_n^{-1}\dot\ell(\btheta_t)\}
  =\Exp\bigg\{\frac{e^{\bbeta_t\tp\x}\z\z\tp}
     {(1+e^{\alpha_t+\bbeta_t\tp\x})^2}\bigg\}
    \longrightarrow\Exp\big(e^{\btheta_t\tp\x}\z\z\tp\big).
\end{align*}
Now we check the Lindeberg-Feller condition. For any $\epsilon>0$,
\begin{align*}
  &\sumn\Exp\Big[\|\{y_i-p_i(\alpha_t,\bbeta_t)\}\z_i\|^2
    I(\|\{y_i-p_i(\alpha_t,\bbeta_t)\}\z_i\|>a_n\epsilon)\Big]\\
  &=n\Exp\Big[\|\{y-p(\btheta_t)\}\z\|^2
    I(\|\{y-p(\btheta_t)\}\z\|>a_n\epsilon)\Big]\\
  &=n\Exp\big[p(\btheta_t)\{1-p(\btheta_t)\}^2\|\z\|^2
     I(\|\{1-p(\btheta_t)\}\z\|>a_n\epsilon)\big]\\
  &\quad+n\Exp\big[\{1-p(\btheta_t)\}\{p(\btheta_t)\}^2\|\z\|^2
    I(\|p(\btheta_t)\z\|>a_n\epsilon)\big]\\
  &\le n\Exp\big[p(\btheta_t)\|\z\|^2I(\|\z\|>a_n\epsilon)\big]
    +n\Exp\big[\{p(\btheta_t)\}^2\|\z\|^2
    I(\|p(\btheta_t)\z\|>a_n\epsilon)\big]\\
  &\le a_n^2\Exp\{e^{\|\beta_t\|\|\x\|}\|\z\|^2I(\|\z\|>a_n\epsilon)\}
    +a_n^2\Exp\{e^{\|\beta_t\|\|\x\|}\|\z\|^2I(\|\z\|>a_n\epsilon)\}\\
  &=o(a_n^2),
\end{align*}
where the last step is from the dominated convergence theorem. Thus, applying the Lindeberg-Feller central limit theorem \citep[Section $^*$2.8 of][]{Vaart:98}, we finish the proof of \eqref{eq:s3}.

The last step is to prove \eqref{eq:s4}. We first show that
\begin{align}
  &\bigg|a_n^{-2}\sumn\phi_i(\btheta_t+a_n^{-1}\Acute\u)\|\z_i\|^2
  -a_n^{-2}\sumn\phi_i(\btheta_t)\|\z_i\|^2\bigg|\notag\\
  &\le a_n^{-2}\sumn\big|\phi_i(\btheta_t+a_n^{-1}\Acute\u)
  -\phi_i(\btheta_t)\big|\|\z_i\|^2\notag\\
  &\le \|a_n^{-1}\Acute\u\|a_n^{-2}\sumn p_i(\btheta_t+a_n^{-1}\Breve{\u})\|\z_i\|^3\notag\\
  &=\frac{\|a_n^{-1}\Acute\u\|}{n}\sumn
    \frac{e^{\x_i\tp\bbeta_t+a_n^{-1}\Breve{\u}\tp\z_i}}
    {\{1+e^{\btheta_t\tp\z_i+a_n^{-1}\Breve{\u}\tp\z_i}\}^2}
    \|\z_i\|^3\notag\\
  &\le\frac{\|a_n^{-1}\Acute\u\|}{n}\sumn
    e^{(\|\bbeta_t\|+\|\u\|)(1+\|\x_i\|)}\|\z_i\|^3
    =\op.\label{eq:s7} 
\end{align}
Here $\Breve{\u}$ lies between $\0$ and $\Acute{\u}$, and thus $\|a_n^{-1}\Breve{\u}\|\le\|\u\|$ for $a_n\ge1$.

To finish the proof, we only need to prove that
\begin{align}
  a_n^{-2}\sumn\phi_i(\btheta_t)\z_i\z_i\tp
 \longrightarrow \Exp(e^{\bbeta_t\tp\x}\z\z\tp),
\end{align}
in probability. This is done by noting that
\begin{align}
  a_n^{-2}\sumn\phi_i(\btheta_t)\z_i\z_i\tp
  =&\frac{1}{ne^{\alpha_t}}\sumn
     \frac{e^{\btheta_t\tp\z_i}}
     {(1+e^{\btheta_t\tp\z_i})^2}\z_i\z_i\tp\\
  =&\frac{1}{n}\sumn\frac{e^{\x_i\tp\bbeta_t}}
     {(1+e^{\btheta_t\tp\z_i})^2}\z_i\z_i\tp
  =\Exp(e^{\bbeta_t\tp\x}\z\z\tp)+\op,
\end{align}
by Proposition 1 of \cite{wang2019more}. %
\end{proof}

\section{Proof of Theorem~\ref{thm:2}}
\begin{proof}[Proof of Theorem~\ref{thm:2}]
The estimator $\htheta\sw$ is the maximizer of $\ell\sw(\btheta)$ defined in \eqref{eq:8}, so $\sqrt{a_n}(\htheta\sw-\theta_t)$ is the maximizer of $\gamma\sw(\u)=\ell\sw(\btheta_t+a_n^{-1}\u)-\ell\sw(\btheta_t)$. By Taylor's expansion,
\begin{align}
  \gamma\sw(\u)
  &=\frac{1}{a_n}\u\tp\dot\ell\sw(\btheta_t)
    +\frac{1}{2a_n^2}\sumn\frac{\delta_i}{\pi_i}
    \phi_i(\btheta_t+a_n^{-1}\Acute\u)(\z_i\tp\u)^2,
\end{align}
where %
\begin{equation*}
  \dot\ell\sw(\btheta)
  =\frac{\partial\ell\sw(\btheta)}{\partial\btheta}
  =\sumn\frac{\delta_i}{\pi_i}\{y_i-p_i(\btheta)\}\z_i
  =\sumn\frac{\delta_i}{\pi_i}\{y_i-p_i(\alpha,\bbeta)\}\z_i
\end{equation*}
is the gradient of $\ell\sw(\btheta)$, and $\Acute\u$ lies between $\0$ and $\u$. 
Similarly to the proof of Theorem~\ref{thm:1}, we only need to show that
\begin{align}\label{eq:s9}
  a_n^{-1}\dot\ell\sw(\btheta_t) \longrightarrow
  \Nor\Big[\0,\ \Exp\big\{
  e^{\bbeta_t\tp\x}(1+ce^{\bbeta_t\tp\x})\z\z\tp\big\}\Big],
\end{align}
in distribution, and for any $\u$,
\begin{align}\label{eq:s10}
  a_n^{-2}\sumn\frac{\delta_i}{\pi_i}
  \phi_i(\btheta_t+a_n^{-1}\Acute\u)\z_i\z_i\tp
  \longrightarrow \Exp\big(e^{\bbeta_t\tp\x}\z\z\tp\big),
\end{align}
in probability. 

We prove~\eqref{eq:s9} first. Recall that $\Dn$ is the full data set and $\delta_i=y_i+(1-y_i)I(u_i\le \pi_0)$, satisfying that
\begin{equation*}
  \pi_i=\Exp(\delta_i|\Dn)=y_i+(1-y_i)\pi_0=\pi_0+(1-\pi_0)y_i.
\end{equation*}
We notice that
\begin{equation*}
  \Exp(\delta_i|\z_i)=p_i(\alpha_t,\bbeta_t) + \{1-p_i(\alpha_t,\bbeta_t)\}\pi_0
  =\pi_0+(1-\pi_0)p_i(\alpha_t,\bbeta_t).
\end{equation*}
Let $\eta_i=\frac{\delta_i}{\pi_i}\{y_i-p_i(\btheta_t)\}\z_i$, we know that $\eta_i$, $i=1, ..., n$, are i.i.d., with the underlying distribution of $\eta_i$ being dependent on $n$. From direction calculation, we have
\begin{align*}
  \Exp(\eta_i|\z_i)
  &=\0,\quad\text{ and}\\ \Var(\eta_i|\z_i)
  &=\Exp\bigg[\frac{\{y_i-p_i(\btheta_t)\}^2}
    {\pi_0+y_i(1-\pi_0)}\bigg|\z_i\bigg]\z_i\z_i\tp\\
  &=\big[p_i(\btheta_t)\{1-p_i(\btheta_t)\}^2
    +\pi_0^{-1}\{1-p_i(\btheta_t)\}\{p_i(\btheta_t)\}^2
    \big]\z_i\z_i\tp\\
  &=\big\{1-p_i(\btheta_t)+\pi_0^{-1}p_i(\btheta_t)\big\}
    p_i(\btheta_t)\{1-p_i(\btheta_t)\}\z_i\z_i\tp\\
  &=\frac{1+\pi_0^{-1}e^{\alpha_t+\x_i\tp\bbeta_t}}
    {(1+e^{\alpha_t+\x_i\tp\bbeta_t})^2}p_i(\btheta_t)\z_i\z_i\tp\\
  &\le e^{\alpha_t}(1+\pi_0^{-1}e^{\alpha_t}e^{\x_i\tp\bbeta_t})
    e^{\x_i\tp\bbeta_t}\z_i\z_i\tp.
\end{align*}
Thus, by the dominated convergence theorem, we obtain that 
\begin{align}
  \Var(\eta_i)=\Exp\{\Var(\eta_i|\z_i)\}
  &=e^{\alpha_t}\Exp\Big\{e^{\x_i\tp\bbeta_t}
    (1+ce^{\x_i\tp\bbeta_t})\z_i\z_i\tp\Big\}\{1+o(1)\}.
\end{align}
Now we check the Lindeberg-Feller condition \citep[Section $^*$2.8 of][]{Vaart:98}. For simplicity, let $\pi=\pi_0+(1-\pi_0)y$ and $\delta=y+(1-y)I(u\le \pi)$, where $u\sim\mathbb{U}(0,1)$. 
For any $\epsilon>0$,
\begin{align*}
  \sumn&\Exp\big\{\|\eta_i\|^2I(\|\eta_i\|>a_n\epsilon)\big\}\\
  =&n\Exp\big[\|\pi^{-1}\delta\{y-p(\btheta_t)\}\z\|^2
    I(\|\pi^{-1}\delta\{y-p(\btheta_t)\}\z\|>a_n\epsilon)\big]\\
  =&\pi_0n\Exp\big[\|\pi^{-1}\{y-p(\btheta_t)\}\z\|^2
    I(\|\pi^{-1}\{y-p(\btheta_t)\}\z\|>a_n\epsilon)\big]\\
  &+(1-\pi_0)n\Exp\big[\pi^{-1}\|y\{y-p(\btheta_t)\}\z\|^2
    I(\|\pi^{-1}y\{y-p(\btheta_t)\}\z\|>a_n\epsilon)\big]\\
  =&\pi_0n\Exp\big[p(\btheta_t)\|\{1-p(\btheta_t)\}\z\|^2
     I(\|\{1-p(\btheta_t)\}\z\|>a_n\epsilon)\big]\\
  &+\pi_0^{-1}n\Exp\big[\{1-p(\btheta_t)\}\|p(\btheta_t)\z\|^2
    I(\pi_0^{-1}\|p(\btheta_t)\z\|>a_n\epsilon)\big]\\
  &+(1-\pi_0)n\Exp\big[p(\btheta_t)\|\{1-p(\btheta_t)\}\z\|^2
    I(\|\{1-p(\btheta_t)\}\z\|>a_n\epsilon)\big]\\
  \le&n\Exp\big\{p(\btheta_t)\|\z\|^2
       I(\|\z\|>a_n\epsilon)\big\}
  +n\pi_0^{-1}\Exp\big\{\|p(\btheta_t)\z\|^2
    I(\|\pi_0^{-1}p(\btheta_t)\z\|>a_n\epsilon)\big\}\\
  \le&ne^{\alpha_t}\Exp\big\{e^{\bbeta_t\tp\x}\|\z\|^2
       I(\|\z\|>a_n\epsilon)\big\}
  +n\pi_0^{-1}e^{2\alpha_t}\Exp\big\{e^{\bbeta_t\tp\x}\|\z\|^2
    I(\pi_0^{-1}e^{\alpha_t}e^{\alpha_t}\|\z\|>a_n\epsilon)\big\}\\
  =&o(ne^{\alpha_t})=o(a_n^2),
\end{align*}
where the second last step is from the dominated convergence theorem and the facts that $a_n\rightarrow\infty$ and $\lim_{n\rightarrow\infty}e^\alpha/\pi_0=c<\infty$. Thus, applying the Lindeberg-Feller central limit theorem \citep[Section $^*$2.8 of][]{Vaart:98} finishes the proof of \eqref{eq:s9}.

Now we prove~\eqref{eq:s10}. By direct calculation, we first notice that
\begin{align}
  \Delta_1\equiv
  a_n^{-2}\sumn\frac{\delta_i}{\pi_i}\phi_i(\btheta_t)\z_i\z_i\tp
  =\onen\sumn\frac{\{y_i+(1-y_i)I(u_i\le \pi_0)\}e^{\x_i\tp\bbeta_t}}
    {\pi_i(1+e^{\alpha_t+\x_i\tp\bbeta_t})^2}\z_i\z_i\tp
\end{align}
has a mean of
\begin{align}\label{eq:s20}
  \Exp(\Delta_1)=
  &\Exp\bigg\{\frac{e^{\bbeta_t\tp\x}}
     {(1+e^{\alpha_t+\bbeta_t\tp\x})^2}\z\z\tp\bigg\}
  =\Exp\big(e^{\bbeta_t\tp\x}\z\z\tp\big)+o(1),
\end{align}
where the last step is by the dominated convergence theorem. In addition, the variance of each component of $\Delta_1$ is bounded by
\begin{align}\label{eq:s21}
  &\onen\Exp\bigg\{\frac{e^{2\bbeta_t\tp\x}\|\z\|^4}
    {\pi(1+e^{\alpha_t+\bbeta_t\tp\x})^4}\bigg\}
    \le\frac{\Exp(e^{2\bbeta_t\tp\x}\|\z\|^4)}{n\pi_0} =o(1),
\end{align}
where the last step is because $ne^{\alpha_t}\rightarrow\infty$ and $e^{\alpha_t}/\pi_0\rightarrow c<\infty$ imply that $n\pi_0\rightarrow\infty$. 
From (\ref{eq:s20}) and (\ref{eq:s21}), Chebyshev's inequality implies that $\Delta_1\rightarrow \Exp\big(e^{\bbeta_t\tp\x}\z\z\tp\big)$ in probability. Notice that
\begin{align*}
  &\bigg|a_n^{-2}\sumn
    \frac{\delta_i}{\pi_i}\phi_i(\btheta_t+a_n^{-1}\Acute\u)\|\z_i\|^2
    -a_n^{-2}\sumn\frac{\delta_i}{\pi_i}\phi_i(\btheta_t)\|\z_i\|^2\bigg|\\
  &\le\|a_n^{-1}\Acute\u\|a_n^{-2}\sumn\frac{\delta_i}{\pi_i}
    p_i(\btheta_t+a_n^{-1}\Breve{\u})\|\z_i\|^3\\
  &\le\|a_n^{-1}\Acute\u\|\times\onen\sumn\frac{\delta_i}{\pi_i}
    e^{(\|\bbeta_t\|+\|\u\|)\|\z_i\|}\|\z_i\|^3
    \equiv\|a_n^{-1}\Acute\u\|\times\Delta_2.
\end{align*}
Since $\|a_n^{-1}\Acute\u\|\rightarrow0$, to finish 
the proof of \eqref{eq:s10}, we only need to prove that $\Delta_2$ is bounded in probability. Using an approach similar to (\ref{eq:s20}) and (\ref{eq:s21}), we can show that $\Delta_2$ has a mean that is bounded and a variance that converges to zero. 

\end{proof}

\section{Proof of Theorem~\ref{thm:3}}
\begin{proof}[Proof of Theorem~\ref{thm:3}]
If we use $\Upsilon_{bc}$ to denote the under-sampled objective function shifted by $\b$, i.e., $\Upsilon_{bc}(\btheta)=\ell\su(\btheta-\b)$, then the estimator $\htheta\sbc$ is the maximizer of
\begin{align}\label{eq:s40}
  \Upsilon_{bc}(\btheta)=\sumn\delta_i
  \big[(\btheta-\b)\tp\z_iy_i
  -\log\{1+e^{(\btheta-\b)\tp\z_i}\}\big].
\end{align}
We notice that $\sqrt{a_n}(\htheta\sbc-\btheta_t)$ is the maximizer of $\gamma_{bc}(\u)=\Upsilon_{bc}(\btheta_t+a_n^{-1}\u)-\Upsilon_{bc}(\btheta_t)$. By Taylor's expansion,
\begin{align}
  \gamma_p(\u)
  &=\frac{1}{a_n}\u\tp\dot\Upsilon_{bc}(\btheta_t)
    +\frac{1}{2a_n^2}\sumn\delta_i
    \phi_i(\btheta_t-\b+a_n^{-1}\Acute\u)(\z_i\tp\u)^2,
\end{align}
where %
\begin{equation*}
  \dot\Upsilon_{bc}(\btheta)
  =\frac{\partial\Upsilon_{bc}(\btheta)}{\partial\btheta}
  =\sumn\delta_i\{y_i-p_i(\btheta_t-\b)\}\z_i
  =\sumn\delta_i\{y_i-p_i(\alpha_t-b,\bbeta_t)\}\z_i
\end{equation*}
is the gradient of $\Upsilon_{bc}(\btheta)$, and $\Acute\u$ lies between $\0$ and $\u$.

Similarly to the proof of Theorem~\ref{thm:1}, we only need to show that
\begin{align}\label{eq:s5}
  a_n^{-1}\dot\Upsilon_{bc}(\btheta_t) \longrightarrow
  \Nor\bigg\{\0,\ \Exp\bigg(
  \frac{e^{\bbeta_t\tp\x}\z\z\tp}
  {1+ce^{\bbeta_t\tp\x}}\bigg)\bigg\},
\end{align}
in distribution, and for any $\u$,
\begin{align}\label{eq:s6}
  a_n^{-2}\sumn\delta_i\phi_i(\btheta_t-\b+a_n^{-1}\Acute\u)\z_i\z_i\tp
  \longrightarrow \Exp\bigg(\frac{e^{\x_i\tp\bbeta_t}}{1+ce^{\x_i\tp\bbeta_t}}
     \z_i\z_i\tp\bigg)
\end{align}
in probability. 

We prove \eqref{eq:s5} first.
Define $\eta_{ui}=\delta_i\{y_i-p_i(\alpha_t-b,\bbeta_t)\}\z_i$.
We have that
\begin{align*}
  \Exp(\eta_{ui}|\z_i)
  &=\Exp[\{\pi_0+y_i(1-\pi_0)\}
    \{y_i-p_i(\alpha_t-b,\bbeta_t)\}|\z_i]\z_i\\
  &=[p_i(\alpha_t,\bbeta_t)\{1-p_i(\alpha_t-b,\bbeta_t)\}
    -\pi_0\{1-p_i(\alpha_t,\bbeta_t)\}
    \{p_i(\alpha_t-b,\bbeta_t)\}]\z_i=0,
\end{align*}
which implies that $\Exp(\eta_{ui})=\0$. For the conditional variance
\begin{align*}
  \Var(\eta_{ui}|\z_i)
  &=\Exp[\{\pi_0+y_i(1-\pi_0)\}
    \{y_i-p_i(\alpha_t-b,\bbeta_t)\}^2|\z_i]\z_i\z_i\tp\\
  &=\big[p_i(\alpha_t,\bbeta_t)\{1-p_i(\alpha_t-b,\bbeta_t)\}^2
    +\pi_0\{1-p_i(\alpha_t,\bbeta_t)\}
    \{p_i(\alpha_t-b,\bbeta_t)\}^2\big]\z_i\z_i\tp\\
  &=\frac{e^{\alpha_t+\x_i\tp\bbeta_t}
    +\pi_0e^{2(\alpha_t-b_0+\x_i\tp\bbeta_t)}}
    {(1+e^{\alpha_t+\x_i\tp\bbeta_t})
    (1+e^{\alpha_t-b_0+\x_i\tp\bbeta_t})^2}\z_i\z_i\tp\\
  &=\frac{e^{\alpha_t+\x_i\tp\bbeta_t}}
    {1+e^{\alpha_t-b_0+\x_i\tp\bbeta_t}}
    \{1-p_i(\alpha_t,\bbeta_t)\}\z_i\z_i\tp
  \le e^{\alpha_t}e^{\x_i\tp\bbeta_t}\z_i\z_i\tp,
\end{align*}
where $e^{\x_i\tp\bbeta_t}\z_i\z_i\tp$ is integrable.
Thus, by the dominated convergence theorem, $\Var(\eta_{ui})$ satisfies that
\begin{align}
  \Var(\eta_{ui})=\Exp\{\Var(\eta_{ui}|\z_i)\}
  &=e^{\alpha_t}\Exp\bigg(\frac{e^{\bbeta_t\tp\x}}
    {1+ce^{\bbeta_t\tp\x}}\bigg)\{1+o(1)\}.
\end{align}
Therefore, we have
\begin{align}
  a_n^{-2}\sumn\Var(\eta_{ui})
    \longrightarrow\Exp\bigg(\frac{e^{\bbeta_t\tp\x}}
    {1+ce^{\bbeta_t\tp\x}}\z\z\tp\bigg).
\end{align}
Now we check the Lindeberg-Feller condition. For any $\epsilon>0$,
\begin{align*}
  &\sumn\Exp\big\{\|\eta_{ui}\|^2I(\|\eta_{ui}\|>a_n\epsilon)\big\}\\
  =&n\Exp\big[\|\delta\{y-p(\btheta_t-\b)\}\z\|^2
    I(\|\delta\{y-p(\btheta_t-\b)\}\z\|>a_n\epsilon)\big]\\
  =&\pi_0n\Exp\big[\|\{y-p(\btheta_t-\b)\}\z\|^2
    I(\|\{y-p(\btheta_t-\b)\}\z\|>a_n\epsilon)\big]\\
  &+(1-\pi_0)n\Exp\big[\|y\{y-p(\btheta_t-\b)\}\z\|^2
    I(\|y\{y-p(\btheta_t-\b)\}\z\|>a_n\epsilon)\big]\\
  =&\pi_0n\Exp\big[p(\btheta_t)\|\{1-p(\btheta_t-\b)\}\z\|^2
     I(\|\{1-p(\btheta_t-\b)\}\z\|>a_n\epsilon)\big]\\
  &+\pi_0n\Exp\big[\{1-p(\btheta_t)\}\|p(\btheta_t-\b)\z\|^2
    I(\|p(\btheta_t-\b)\z\|>a_n\epsilon)\big]\\
  &+(1-\pi_0)n\Exp\big[p(\btheta_t)\|\{1-p(\btheta_t-\b)\}\z\|^2
    I(\|\{1-p(\btheta_t-\b)\}\z\|>a_n\epsilon)\big]\\
  \le&n\Exp\big\{p(\btheta_t)\|\z\|^2I(\|\z\|>a_n\epsilon)\big\}
  +\pi_0n\Exp\big[\|p(\btheta_t-\b)\z\|^2
    I(\|\z\|>a_n\epsilon)\big]\\
  \le&ne^{\alpha_t}\Exp\big\{e^{\bbeta_t\tp\x}\|\z\|^2
       I(\|\z\|>a_n\epsilon)\big\}
  +\pi_0^{-1}ne^{2\alpha_t}\Exp\big\{e^{2\bbeta_t\tp\x}\|\z\|^2
    I(\|\z\|>a_n\epsilon)\big\}\\
  =&o(ne^{\alpha_t})=o(a_n^2),
\end{align*}
where the second last step is from the dominated convergence theorem. Thus, applying the Lindeberg-Feller central limit theorem \citep[Section $^*$2.8 of][]{Vaart:98} finishes the proof of \eqref{eq:s5}.

No we prove \eqref{eq:s6}. First, letting
\begin{align}
  \Delta_3\equiv a_n^{-2}\sumn\delta_i\phi_i(\btheta_t-\b)\z_i\z_i\tp
  =\frac{1}{n}\sumn
    \frac{\{y_i+(1-y_i)I(u_i\le \pi_0)\}e^{-b_0+\x_i\tp\bbeta_t}}
     {\{1+e^{\alpha_t-b_0+\x_i\tp\bbeta_t}\}^2}\z_i\z_i\tp,
\end{align}
the mean of $\Delta_3$ satisfies that
\begin{align}
  \Exp(\Delta_3)
  =&\Exp\bigg[
     \frac{e^{\bbeta_t\tp\x}}
     {\{1+e^{\alpha_t+\bbeta_t\tp\x}\}\{1+e^{\alpha_t-b_0+\bbeta_t\tp\x}\}}
     \z\z\tp\bigg]
  =\Exp\bigg(\frac{e^{\bbeta_t\tp\x}}{1+ce^{\bbeta_t\tp\x}}
     \z\z\tp\bigg)+o(1),
\end{align}
by the dominated convergence theorem, and the variance of each component of $\Delta_3$ is bounded by
\begin{align}
  \onen\Exp\bigg[
    \frac{\{y+(1-y)I(u\le \pi_0)\}e^{-2b_0+2\bbeta_t\tp\x}}
     {\{1+e^{\alpha_t-b_0+\bbeta_t\tp\x}\}^4}\|\z\|^4\bigg]
  \le\frac{\Exp(e^{2\bbeta_t\tp\x}\|\z\|^4)}{n\pi_0}=o(1).
\end{align}
Thus, Chebyshev's inequality implies that
\begin{equation}\label{eq:s22}
  \Delta_3\longrightarrow
  \Exp\bigg(\frac{e^{\bbeta_t\tp\x}}{1+ce^{\bbeta_t\tp\x}}\z\z\tp\bigg),
\end{equation}
in probability. Furthermore,
\begin{align}
  &\bigg|a_n^{-2}\sumn
    \delta_i\phi_i(\btheta_t-\b+a_n^{-1}\Acute\u)\|\z_i\|^2
    -a_n^{-2}\sumn\delta_i\phi_i(\btheta_t-\b)\|\z_i\|^2\bigg|\notag\\
  &\le \|a_n^{-1}\Acute\u\|a_n^{-2}\sumn\delta_i
    p_i(\btheta_t-\b+a_n^{-1}\Breve{\u})\|\z_i\|^3\notag\\
  &\le\frac{\|a_n^{-1}\Acute\u\|}{n}\sumn\frac{\delta_i}{\pi_0}
    e^{(\|\bbeta_t\|+\|\u\|)(1+\|\x_i\|)}\|\z_i\|^3
    \equiv\|a_n^{-1}\Acute\u\|\times\Delta_4=\op,\label{eq:s23}
\end{align}
where the last step is because $\Delta_4$ is bounded in probability due to the fact that it has a mean that is bounded and a variance that converges to zero. 
Combing (\ref{eq:s22}) and (\ref{eq:s23}), \eqref{eq:s6} follows.
\end{proof}

\section{Proof of Proposition~\ref{prop1}}
\begin{proof}[Proof of Proposition~\ref{prop1}]
  Let
  \begin{equation*}
    \g=\frac{1}{\sqrt{h}}\big\{\Exp(h^{-1}\v\v\tp)\big\}^{-1}\v
    -\sqrt{h}\big\{\Exp(\v\v\tp)\big\}^{-1}\v.
  \end{equation*}
  Since $\g\g\tp\ge\0$, we have
  \begin{align*}
    \0\le\Exp(\g\g\tp)=
    \big\{\Exp(\v\v\tp)\big\}^{-1}\Exp(h\v\v\tp)
    \big\{\Exp(\v\v\tp)\big\}^{-1}
    -\big\{\Exp(h^{-1}\v\v\tp)\big\}^{-1},
  \end{align*}
which finishes the proof. 
\end{proof}

\section{Proof of Theorem~\ref{thm:4}}
\begin{proof}[Proof of Theorem~\ref{thm:4}]
The estimator $\htheta\ow$ is the maximizer of \eqref{eq:11}, so $\sqrt{a_n}(\htheta\ow-\theta_t)$ is the maximizer of $\gamma\ow(\u)=\ell\ow(\btheta_t+a_n^{-1}\u)-\ell\ow(\btheta_t)$. By Taylor's expansion,
\begin{equation}
  \gamma\ow(\u)
  =\frac{1}{a_n}\u\tp\dot\ell\ow(\btheta_t)
    +\frac{1}{2a_n^2}\sumn\frac{\tau_i}{w_i}
    \phi_i(\btheta_t+a_n^{-1}\Acute\u)(\z_i\tp\u)^2,
\end{equation}
where %
\begin{equation*}
  \dot\ell\ow(\btheta)
  =\frac{\partial\ell\ow(\btheta)}{\partial\btheta}
  =\sumn\frac{\tau_i}{w_i}\{y_i-p_i(\btheta_t)\}\z_i
  =\sumn\frac{\tau_i}{w_i}\{y_i-p_i(\alpha_t-b,\bbeta_t)\}\z_i
\end{equation*}
is the gradient of $\ell\ow(\btheta)$, and $\Acute\u$ lies between $\0$ and $\u$.
Similarly to the proof of Theorem~\ref{thm:1}, we only need to show that
\begin{align}\label{eq:s12}
  a_n^{-1}\dot\ell\ow(\btheta_t) \longrightarrow
  \Nor\Big\{\0,\
  \frac{(1+\lambda)^2+\lambda}{(1+\lambda)^2}
    \Exp\big(e^{\bbeta_t\tp\x}\z\z\tp\big)\Big\},
\end{align}
in distribution, and for any $\u$,
\begin{align}\label{eq:s13}
  a_n^{-2}\sumn\frac{\tau_i}{w_i}
  \phi_i(\btheta_t+a_n^{-1}\Acute\u)\z_i\z_i\tp
  \longrightarrow \Exp\big(e^{\bbeta_t\tp\x}\z\z\tp\big),
\end{align}
in probability.

We prove~\eqref{eq:s12} first. %
Denote $\eta_{owi}=\tau_iw_i^{-1}\{y_i-p_i(\btheta_t)\}\z_i$, so $\eta_{owi}$, $i=1, ..., n$, are i.i.d. with the underlying distribution of $\eta_{owi}$ being dependent on $n$. From direction calculation, we have
\begin{align*}
  \Exp(\eta_{owi}|\z_i)
  &=\0,\quad\text{ and}\\
  \Var(\eta_{owi}|\z_i)
  &=\Exp\bigg[\frac{\{y_i(3\lambda_n+\lambda_n^2)+1\}\{y_i-p_i(\btheta_t)\}^2}
    {(1+\lambda_n y_i)^2}\bigg|\z_i\bigg]\z_i\z_i\tp\\
  &=\bigg[p_i(\btheta_t)\{1-p_i(\btheta_t)\}^2
    \frac{(1+\lambda_n)^2+\lambda_n}{(1+\lambda_n)^2}
    +\{1-p_i(\btheta_t)\}\{p_i(\btheta_t)\}^2\bigg]\z_i\z_i\tp\\
  &=\frac{(1+\lambda_n)^2+\lambda_n}{(1+\lambda_n)^2}
    e^{\alpha_t}e^{\bbeta_t\tp\x}\z_i\z_i\tp\{1+\op\},
\end{align*}
where the $\op$ is bounded. Thus, by the dominated convergence theorem, we obtain that
\begin{align*}
  \Var(\eta_{owi})
  &=\frac{(1+\lambda)^2+\lambda}{(1+\lambda)^2}
    e^{\alpha_t}\Exp\big(e^{\x\tp\bbeta_t}\z\z\tp\big)\{1+o(1)\}.
\end{align*}
Now we check the Lindeberg-Feller condition \citep[Section $^*$2.8 of][]{Vaart:98}. Let $w=1+\lambda_ny$ and $\tau=yv+1$, where $v\sim \mathbb{POI}(\lambda_n)$. For any $\epsilon>0$,
\begin{align*}
  &\sumn\Exp\big[\|\eta_{owi}\|^2I(\|\eta_{owi}\|>a_n\epsilon)\big]\\
  &=n\Exp\big[\|w^{-1}\tau\{y-p(\btheta_t)\}\z\|^2
    I(\|w^{-1}\tau\{y-p(\btheta_t)\}\z\|>a_n\epsilon)\big]\\
  &\le\frac{n}{a_n\epsilon}
    \Exp\big[\|w^{-1}\tau\{y-p(\btheta_t)\}\z\|^3\\
  &=\frac{n}{a_n\epsilon}
    \Exp\bigg[\frac{(1+vy)^3}{(1+\lambda_n y)^3}
    \{y-p(\btheta_t)\}^3\|\z\|^3\bigg]\\
  &\le\frac{n}{a_n\epsilon}
    \frac{1+7\lambda_n+6\lambda_n^2+\lambda_n^3}{(1+\lambda_n)^3}
    \Exp\{p(\btheta_t)\|\z\|^3\}
    +\frac{n}{a_n\epsilon}\Exp[\{p(\btheta_t)\}^3\|\z\|^3]\\
  &\le\frac{a_n}{\epsilon}
    \frac{1+7\lambda_n+6\lambda_n^2+\lambda_n^3}{(1+\lambda_n)^3}
    \Exp(e^{\x_i\tp\bbeta_t}\|\z\|^3)
    +\frac{a_ne^{2\alpha_t}}{\epsilon}
    \Exp(e^{3\x_i\tp\bbeta_t}\|\z\|^3)=o(a_n^2).
\end{align*}
Thus, applying the Lindeberg-Feller central limit theorem \citep[Section $^*$2.8 of][]{Vaart:98} finishes the proof of \eqref{eq:s12}.

Now we prove~\eqref{eq:s13}. Let
\begin{align*}
  \Delta_5\equiv
  a_n^{-2}\sumn\frac{\tau_i}{w_i}\phi_i(\btheta_t)\z_i\z_i\tp
  =\onen\sumn\frac{\tau_i}{w_i}\frac{e^{\x_i\tp\bbeta_t}}
  {(1+e^{\alpha_t+\x_i\tp\bbeta_t})^2}\z_i\z_i\tp.
\end{align*}
Since
\begin{align*}
  \Exp(\Delta_5)=
  &\Exp\bigg\{\frac{e^{\bbeta_t\tp\x}}
     {(1+e^{\alpha_t+\bbeta_t\tp\x})^2}\z\z\tp\bigg\}
  =\Exp\big(e^{\bbeta_t\tp\x}\z\z\tp\big)+o(1),
\end{align*}
by the dominated convergence theorem, and each component of $\Delta_5$ has a variance that is bounded by
\begin{align*}
  &\onen\Exp\bigg\{\frac{2e^{2\bbeta_t\tp\x}\|\z\|^4}
    {(1+e^{\alpha_t+\bbeta_t\tp\x})^4}\bigg\}
    \le\frac{2\Exp(e^{2\bbeta_t\tp\x}\|\z\|^4)}{n}
    =o(1),
\end{align*}
applying Chebyshev's inequality gives that
\begin{equation*}
  \Delta_5\longrightarrow\Exp\big(e^{\bbeta_t\tp\x}\z\z\tp\big),
\end{equation*}
in probability. Thus, \eqref{eq:s13} follows from the fact that
\begin{align*}
  &\bigg|a_n^{-2}\sumn
    \frac{\tau_i}{w_i}\phi_i(\btheta_t+a_n^{-1}\Acute\u)\|\z_i\|^2
    -a_n^{-2}\sumn\frac{\tau_i}{w_i}\phi_i(\btheta_t)\|\z_i\|^2\bigg|\\
  &\le\|a_n^{-1}\Acute\u\|a_n^{-2}\sumn\frac{\tau_i}{w_i}
    p_i(\btheta_t+a_n^{-1}\Breve{\u})\|\z_i\|^3\\
  &\le\frac{\|a_n^{-1}\Acute\u\|}{n}\sumn\frac{\tau_i}{w_i}
    e^{(\|\bbeta_t\|+\|\u\|)\|\z_i\|}\|\z_i\|^3
    =\op,
\end{align*}
where the last step is because $n^{-1}\sumn\tau_iw_i^{-1}
e^{(\|\bbeta_t\|+\|\u\|)\|\z_i\|}\|\z_i\|^3$ has a bounded mean and a bounded variance and thus it is bounded in probability. 
\end{proof}

\section{Proof of Theorem~\ref{thm:5}}
\begin{proof}[\bf Proof of Theorem~\ref{thm:5}]
The over-sampled estimator $\htheta\obc$ is the maximizer of
\begin{equation}\label{eq:s40}
  \Upsilon_{oc}(\btheta)=\frac{1}{1+\lambda_n}
  \sumn\tau_i\big[(\btheta+\b_o)\tp\z_iy_i
  -\log\{1+e^{\z_i\tp(\btheta+\b_o)}\}\big].
\end{equation}
Thus, $\sqrt{a_n}(\htheta\obc-\btheta_t)$ is the maximizer of $\gamma_{oc}(\u)=\Upsilon_{oc}(\btheta_t+a_n^{-1}\u)-\Upsilon_{oc}(\btheta_t)$. By Taylor's expansion,
\begin{equation}
  \gamma_{oc}(\u)
  =\frac{1}{a_n}\u\tp\dot\Upsilon_{oc}(\btheta_t)
    +\frac{1}{2a_n^2(1+\lambda_n)}\sumn\tau_i
    \phi_i(\btheta_t+\b_o+a_n^{-1}\Acute\u)(\z_i\tp\u)^2,
  \end{equation}
where %
\begin{equation*}
\dot\Upsilon_{oc}(\btheta)
=\frac{\partial\Upsilon_{oc}(\btheta)}{\partial\btheta}
=\frac{1}{1+\lambda_n}\sumn\tau_i\{y_i-p_i(\alpha_t+b_{o0},\bbeta_t)\}\z_i
\end{equation*}
is the gradient of $\Upsilon_{oc}(\btheta)$, and $\Acute\u$ lies between $\0$ and $\u$.

Similarly to the proof of Theorem~\ref{thm:1}, we only need to show that
\begin{align}\label{eq:s14}
  a_n^{-1}\dot\Upsilon_{oc}(\btheta_t) \longrightarrow
  \Nor\bigg[\0,\ \frac{(1+\lambda)^2+\lambda}{(1+\lambda)^2}
  \ \Exp\bigg\{\frac{e^{\bbeta_t\tp\x}}
    {(1+\cm e^{\bbeta_t\tp\x})^2}\z\z\tp\bigg\}\bigg],
\end{align}
in distribution, and for any $\u$,
\begin{align}\label{eq:s15}
  \frac{1}{a_n^2(1+\lambda_n)}
  \sumn\tau_i\phi_i(\btheta_t+\b_o+a_n^{-1}\Acute\u)\z_i\z_i\tp
  \longrightarrow \Exp\bigg(\frac{e^{\bbeta_t\tp\x}}
    {1+\cm e^{\bbeta_t\tp\x}}\z\z\tp\bigg),
\end{align}
in probability.

We prove \eqref{eq:s14} first. Let $\eta_{obi}=(1+\lambda_n)^{-1}\tau_i\{y_i-p_i(\alpha_t+b_{o0},\bbeta_t)\}\z_i$.
We have that
\begin{align*}
  (1+\lambda_n)\Exp(\eta_{obi}|\z_i)
  &=\Exp[(1+\lambda_n y_i)\{y_i-p_i(\alpha_t+b_{o0},\bbeta_t)\}|\z_i]\z_i\\
  &=[p_i(\alpha_t,\bbeta_t)(1+\lambda_n)
    \{1-p_i(\alpha_t+b_{o0},\bbeta_t)\}\\
  &\quad  -\{1-p_i(\alpha_t,\bbeta_t)\}
    \{p_i(\alpha_t+b_{o0},\bbeta_t)\}]\z_i=0,
\end{align*}
which implies that $\Exp(\eta_{obi})=\0$. For the conditional variance
\begin{align*}
  &(1+\lambda_n)^2\Var(\eta_{obi}|\z_i)\\
  &=\Exp[\{1+3\lambda_n y_i+\lambda_n^2 y_i\}
    \{y_i-p_i(\alpha_t+b_{o0},\bbeta_t)\}^2|\z_i]\z_i\z_i\tp\\
  &=\big[p_i(\alpha_t,\bbeta_t)(1+3\lambda_n+\lambda_n^2)
    \{1-p_i(\alpha_t+b_{o0},\bbeta_t)\}^2
    +\{1-p_i(\alpha_t,\bbeta_t)\}
    \{p_i(\alpha_t+b_{o0},\bbeta_t)\}^2\big]\z_i\z_i\tp\\
  &=\frac{(1+3\lambda_n+\lambda_n^2)e^{\alpha_t+\x_i\tp\bbeta_t}
    +e^{2(\alpha_t+b_{o0}+\x_i\tp\bbeta_t)}}
    {(1+e^{\alpha_t+\x_i\tp\bbeta_t})
    (1+e^{\alpha_t+b_{o0}+\x_i\tp\bbeta_t})^2}\z_i\z_i\tp\\
  &=\frac{(1+3\lambda_n+\lambda_n^2)e^{\alpha_t+\x_i\tp\bbeta_t}}
    {(1+e^{\alpha_t+b_{o0}+\x_i\tp\bbeta_t})^2}
    \frac{1+\frac{1+2\lambda_n+\lambda_n^2}{1+3\lambda_n+\lambda_n^2}
    e^{\alpha_t+\x_i\tp\bbeta_t}}{1+e^{\alpha_t+\x_i\tp\bbeta_t}}\z_i\z_i\tp\\
  &=\frac{(1+3\lambda_n+\lambda_n^2)e^{\alpha_t+\x_i\tp\bbeta_t}}
    {(1+e^{\alpha_t+b_{o0}+\x_i\tp\bbeta_t})^2}\z_i\z_i\tp\{1+\op\}\\
  &=e^{\alpha_t}(1+3\lambda_n+\lambda_n^2)
    \frac{e^{\x_i\tp\bbeta_t}}{(1+\cm e^{\x_i\tp\bbeta_t})^2}
    \z_i\z_i\tp\{1+\op\},
\end{align*}
where the $\op$'s above are all bounded and the last step is because $(1+\lambda_n)e^{\alpha_t}\rightarrow\cm$. 
Thus, by the dominated convergence theorem, $\Var(\eta_{obi})$ satisfies that
\begin{align}
  \Var(\eta_{obi})
  &=e^{\alpha_t}\frac{(1+\lambda)^2+\lambda}{(1+\lambda)^2}
    \Exp\bigg\{\frac{e^{\bbeta_t\tp\x}}
    {(1+\cm e^{\bbeta_t\tp\x})^2}\bigg\}\{1+o(1)\},
\end{align}
which indicates that
\begin{align}
  \frac{1}{a_n^2}\sumn\Var(\eta_{obi})
  \longrightarrow\frac{(1+\lambda)^2+\lambda}{(1+\lambda)^2}
  \ \Exp\bigg\{\frac{e^{\bbeta_t\tp\x}}
    {(1+\cm e^{\bbeta_t\tp\x})^2}\z\z\tp\bigg\}.
\end{align}
Now we check the Lindeberg-Feller condition. Recall that $\tau=yv+1$, where $v\sim \mathbb{POI}(\lambda_n)$. We can show that $\Exp\{(1+v)^3\}<2(1+\lambda_n)^3$. For any $\epsilon>0$,
\begin{align*}
  &a_n\epsilon(1+\lambda_n)^3
    \sumn\Exp\big\{\|\eta_{obi}\|^2I(\|\eta_{obi}\|>a_n\epsilon)\big\}
  \le (1+\lambda_n)^3\sumn\Exp(\|\eta_{obi}\|^3)\\
  &= n\Exp\big[\|\tau^3\{y-p(\btheta_t+\b_o)\}\z\|^3\big]\\
  &=n\Exp\big[p(\btheta_t)(1+v)^3\|\{1-p(\btheta_t+\b_o)\}\z\|^3\big]
  +n\Exp\big[\{1-p(\btheta_t)\}\|p(\btheta_t+\b_o)\z\|^3\big]\\
  &\le 2n(1+\lambda_n)^3\Exp\big\{p(\btheta_t)\|\z\|^3\big\}
  +n\Exp\big\{\|p(\btheta_t+\b_o)\z\|^3\big\}\\
  &\le 2n(1+\lambda_n)^3e^{\alpha_t}\Exp\big(e^{\bbeta_t\tp\x}\|\z\|^3\big)
  +n(1+\lambda_n)^3e^{3\alpha_t}\Exp\big(e^{3\bbeta_t\tp\x}\|\z\|^3\big)\\
  &=(1+\lambda_n)^3 O(a_n^2).
\end{align*}
This indicates that $a_n^{-2}\sumn\Exp\{\|\eta_{obi}\|^2I(\|\eta_{obi}\|>a_n\epsilon)\}=o(1)$, and thus the Lindeberg-Feller condition holds. Applying the Lindeberg-Feller central limit theorem \citep[Section $^*$2.8 of][]{Vaart:98} finishes the proof of \eqref{eq:s14}.

No we prove \eqref{eq:s15}. Let
\begin{align}
  \Delta_6
  \equiv\frac{1}{a_n^2(1+\lambda_n)}
  \sumn\tau_i\phi_i(\btheta_t+\b_o)\z_i\z_i\tp
  =\frac{1}{n}\sumn
    \frac{(1+v_iy_i)e^{\x_i\tp\bbeta_t}}
     {\{1+e^{\alpha_t+b_{o0}+\x_i\tp\bbeta_t}\}^2}\z_i\z_i\tp.
\end{align}
Note that
\begin{align}
  \Exp(\Delta_6)
  &=\Exp\bigg\{\frac{(1+\lambda_n y)e^{\bbeta_t\tp\x}}
     {(1+e^{\alpha_t+b_{o0}+\bbeta_t\tp\x})^2}\z\z\tp\bigg\}\\
  &=\Exp\bigg\{\frac{e^{\bbeta_t\tp\x}}
    {(1+e^{\alpha_t+\bbeta_t\tp\x})
    (1+e^{\alpha_t+b_{o0}+\bbeta_t\tp\x})}\z\z\tp\bigg\}\\
  &=\Exp\bigg(\frac{e^{\bbeta_t\tp\x}}
    {1+\cm e^{\bbeta_t\tp\x}}\z\z\tp\bigg)+o(1),
\end{align}
by the dominated convergence theorem, and the variance of each component of $\Delta_6$ is bounded by
\begin{align*}
  &\onen\Exp\bigg[
    \frac{(1+vy)^2e^{2\bbeta_t\tp\x}}
     {\{1+e^{\alpha_t+b_{o0}+\bbeta_t\tp\x}\}^4}\|\z\|^4\bigg]\\
  &=\onen\Exp\bigg[
    \frac{\{1+(3\lambda_n+\lambda_n^2)p(\btheta_t)\}e^{2\bbeta_t\tp\x}}
     {\{1+e^{\alpha_t+b_{o0}+\bbeta_t\tp\x}\}^4}\|\z\|^4\bigg]\\
  &\le\frac{\Exp(e^{2\bbeta_t\tp\x}\|\z\|^4)}{n}+
    \frac{e^{\alpha_t}(3\lambda_n+\lambda_n^2)}{n}
    \Exp\big(e^{3\bbeta_t\tp\x}\|\z\|^4\big)=o(1),
\end{align*}
where the last step is because $n^{-1}e^{\alpha_t}\lambda_n^2=(e^{\alpha_t}\lambda_n)^2a_n^{-2}\rightarrow0$ and both expectations are finite.
Therefore, Chebyshev's inequality implies that $\Delta_6\rightarrow0$ in probability. Thus, \eqref{eq:s15} follows from the fact that
\begin{align*}
  &\bigg|\frac{1}{a_n^2(1+\lambda_n)}\sumn
    \tau_i\phi_i(\btheta_t+\b_o+a_n^{-1}\Acute\u)\|\z_i\|^2
    -\frac{1}{a_n^2(1+\lambda_n)}
    \sumn\tau_i\phi_i(\btheta_t+\b_o)\|\z_i\|^2\bigg|\\
  &\le\frac{\|a_n^{-1}\Acute\u\|}{a_n^2(1+\lambda_n)}\sumn\tau_i
    p_i(\btheta_t+\b_o+a_n^{-1}\Breve{\u})\|\z_i\|^3\\
  &\le\frac{\|a_n^{-1}\Acute\u\|}{n}\sumn
    (1+v_iy_i)e^{(\|\bbeta_t\|+\|\u\|)\|\z_i\|}\|\z_i\|^3
    =\op,
\end{align*}
where the last step is from the fact that $n^{-1}\sumn(1+v_iy_i)e^{(\|\bbeta_t\|+\|\u\|)\|\z_i\|}\|\z_i\|^3$ has a bounded mean and a bounded variance, and an application of Chebyshev's inequality.
\end{proof}

\bibliographystyle{natbib}
\bibliography{logisticref}

\end{document}